\documentclass[11pt]{article} 

\usepackage{amsmath,amsthm,amssymb}
\usepackage{natbib}
\renewcommand{\citealp}[1]{\citet{#1}}
\usepackage{graphicx}
\usepackage{fullpage}

\usepackage[utf8]{inputenc} 
\usepackage[T1]{fontenc}    
\usepackage{hyperref}       
\usepackage{url}            
\usepackage{booktabs}       
\usepackage{amsfonts}       
\usepackage{nicefrac}       
\usepackage{microtype}      
\usepackage{nccmath}

\usepackage{xspace}
\usepackage{xcolor}  
\usepackage[capitalize]{cleveref}
\crefname{theorem}{Theorem}{Theorems}
\crefname{lemma}{Lemma}{Lemmas}
\crefname{corollary}{Corollary}{Corollaries}

\crefname{definition}{Definition}{Definitions}
\crefname{assumption}{Assumption}{Assumptions}

\crefname{remark}{Remark}{Remarks}

\crefname{figure}{Figure}{Figures}

\crefname{algocf}{Algorithm}{Algorithms}

\newtheorem{theorem}{Theorem}[section]
\newtheorem{lemma}[theorem]{Lemma}
\newtheorem{corollary}[theorem]{Corollary}
\newtheorem{remark}[theorem]{Remark}

\theoremstyle{definition}

\theoremstyle{definition}

\usepackage[disable]{todonotes}

\usepackage[ruled]{algorithm2e}
\DontPrintSemicolon

\newcommand{\cH}{\mathcal{H}}

\newcommand{\bE}{\mathbb{E}}
\newcommand{\bR}{\mathbb{R}}

\newcommand{\bI}{\mathbb{I}}
\newcommand{\Reals}{\bR}
\newcommand{\ExpThree}{\texttt{Exp3}\xspace}
\newcommand{\adaexp}{\texttt{DAda-\ExpThree}\xspace}
\newcommand{\dedaexp}{\texttt{DeDa-\ExpThree}\xspace}

\newcommand{\EE}[1]{\bE\left[ #1\right]}

\newcommand{\I}[1]{\bI\left[ #1\right]}

\newcommand{\tl}{\tilde}
\newcommand{\dt}{d}
\newcommand{\tdt}{\tl{\dt}}
\newcommand{\dtau}{\tl{\tau}}
\newcommand{\dD}{\tl{D}}

\newcommand{\Lbck}{L^{\textnormal{bck}}}
\newcommand{\lbck}{\ell^{\textnormal{bck}}}
\newcommand{\lfwd}{\ell^{\textnormal{fwd}}}
\newcommand{\hell}{\hat{\ell}}
\newcommand{\tell}{\tl{\ell}}

\title{Adapting to Delays and Data in Adversarial Multi-Armed Bandits}
\usepackage{times}

\author{%
  Andr\'as Gy\"orgy \ \ \ \ \ Pooria Joulani \\
  \\
  \small DeepMind, London, UK \\
  \small \nolinkurl{{agyorgy,pjoulani}@google.com}
}

\date{} 

\begin{document}

\maketitle

\begin{abstract}
     We consider the adversarial multi-armed bandit problem under delayed feedback. We analyze variants of the \ExpThree algorithm that tune their step-size using only information (about the losses and delays) available at the time of the decisions, and obtain regret guarantees that adapt to the observed (rather than the worst-case) sequences of delays and/or losses. First, through a remarkably simple proof technique, we show that with proper tuning of the step size, the algorithm achieves an optimal (up to logarithmic factors) regret of order $\sqrt{\log(K)(TK + D)}$ both in expectation and in high probability, where $K$ is the number of arms, $T$ is the time horizon, and $D$ is the cumulative delay. The high-probability version of the bound, which is the first high-probability delay-adaptive bound in the literature, crucially depends on the use of implicit exploration in estimating the losses. Then, following \citet{ZiSe19}, we extend these results so that the algorithm can ``skip'' rounds with large delays, resulting in regret bounds of order $\sqrt{TK\log(K)} + |R| + \sqrt{D_{\bar{R}}\log(K)}$, where $R$ is an arbitrary set of rounds (which are skipped) and $D_{\bar{R}}$ is the cumulative delay of the feedback for other rounds. Finally, we present another, data-adaptive (AdaGrad-style) version of the algorithm for which the regret adapts to the observed (delayed) losses instead of only adapting to the cumulative delay (this algorithm requires an a priori upper bound on the maximum delay, or the advance knowledge of the delay for each decision when it is made). The resulting bound can be orders of magnitude smaller on benign problems, and it can be shown that the delay only affects the regret through the loss of the best arm.
\end{abstract}

\section{Introduction}

The multi-armed bandit problem is a canonical model for sequential decision making with limited feedback. In this model a learner makes a sequence of actions. After every action, the learner immediately observes the loss corresponding to its action. On the other hand, in practical applications of bandit algorithms, such as in recommender systems or display advertisement, the loss feedback to the algorithm may be severely delayed, as, for example, the system needs to serve several other users before a user's feedback becomes available.
For this reason, several papers in the literature considered a delayed version of this problem.

In this paper we consider the adversarial version of the bandit problem with an oblivious adversary. Given a set of $K$ actions and a time horizon $T$, it is well known that the worst-case regret achievable by a learner is of order $\sqrt{KT}$ \citep{AuBu09}. The delayed setting was perhaps first considered by \citet{NeGyAnSze10,NeGySzA14}, who showed that in case every feedback is delayed by a constant $\dt$, the \ExpThree algorithm \citep{ACFS:2002} achieves a regret of $O(\sqrt{dKT\log(K)})$. This result was extended to the general partial monitoring case by \citet{Joulani:2013}. The next important step was made by \citet{CeGeMa19}, who showed that the effect of the delay and the number of arms is in fact not intertwined: they proved that the worst-case regret is at least $\Omega(\max\{\sqrt{KT},\sqrt{dT\log(K)}\})$ (for $d \le T/\log K$), and that the \ExpThree algorithm achieves this (up to a logarithmic factor). These bounds show that, at least in the case of fixed delays, it is possible to achieve a regret that scales with the cumulative delay $D=dT$. In the full information case, \citet{QuKh15} were the first to show that it is possible to achieve a regret that scales with the cumulative delay (defined as the sum of the delays) in case of non-uniform delays, that is, when the delay for the different time steps can be different, and showed an optimal regret of order $O(\sqrt{(D+T)\log(K)})$, where $D$ is the sum of the (arbitrary, not necessarily equal) delays. This result was strengthened by \citet{JoulaniGS16}, who showed that this can be done in a fully adaptive way, without prior knowledge on the delays, and without resorting to the doubling trick. 

Thinking along similar lines, \citet{CeGeMa19} posed the question of whether a regret growing with the cumulative delay is achievable for arbitrary delays in the bandit case, more precisely, if a $\sqrt{KT}+\sqrt{D\log(K)}$ regret is achievable. 
Recently, \citet{thune19} gave an algorithm which achieves essentially the same bound but with an oracle tuning depending on the cumulative delay $D$, or with the advance knowledge of the delays at the time of the action.
At the same time, \citet{delayed_exp3_wrong} tried to avoid the advance knowledge of the delays through a doubling trick, but obtained a regret bound with a sub-optimal dependence on $K$ (i.e, $K$ rather than $\sqrt{K}$).%
\footnote{\label{fn:difficulty}From a technical perspective, the difficulties in the works of \citet{CeGeMa19,thune19,delayed_exp3_wrong}, arise because the analysis technique they adopt requires bounding a hard-to-control ``drift'' term. \citet{thune19} control this through a bound that requires the step-size to be diminished using the knowledge of the total or upcoming delay (this is also used by \citealp{CeGeMa19} as they consider the case of fixed, known delays), while \citet{delayed_exp3_wrong} also need to bound a similar term
(cf. the derivation of their Eq.~36),
and end up requiring either an oracle tuning (similar to \citealp{CeGeMa19,thune19}), or a doubling trick.
}

More recently, \citet{ZiSe19} achieved an optimal $O(\sqrt{KT}+\sqrt{D\log(K)})$ bound (optimal in terms of the cumulative delay $D$) with an anytime algorithm that requires no advance knowledge about the delays. On the other hand, \citet{thune19} pointed out that the scaling of the regret with the cumulative delay $D$ can be quite pessimistic in certain cases (e.g., if the feedback of the first round is missing until the very end but no other feedback is delayed, the resulting cumulative delay is $D=T$, which seems an unreasonably large price to pay), and proposed to ``skip'' rounds with excessive delays, and aim for  regret bounds where the $\sqrt{D\log(K)}$ term is replaced with $|R| + \sqrt{D_{\bar{R}}\log(K)}$, where $R$ is an arbitrary set of rounds (which are skipped) and $D_{\bar{R}}$ is the cumulative delay of the feedback for other rounds. While they achieved this bound only using the knowledge of the delay for every prediction made, the method of  \citet{ZiSe19} achieves this goal only assuming that the delays become known when the feedback arrives.

The analysis of \citet{ZiSe19} (like ours) uses a follow-the-regularized-leader (FTRL) approach, but like the other papers mentioned above, requires specializing the FTRL analysis (respectively, the analysis of \ExpThree in the works of \citealp{CeGeMa19,thune19,delayed_exp3_wrong}) to handle the effect of delays on the updates, and hence repeating the main analysis steps from scratch. In addition, their modified FTRL analysis is specialized to a relatively complicated regularizer to avoid the $\sqrt{\log(K)}$ term in the \ExpThree bound (hence their update cannot be computed in closed form), leaving the simple \ExpThree case unattended.

\subsection{Contributions} 

In this paper we are concerned with similar, fully delay-adaptive methods, based on different versions of the \ExpThree algorithm, and derive several novel results for the adversarial bandit problem with delayed feedback:
\begin{itemize}
\item Using a remarkably simple proof technique, we derive the first proper step-size tuning of the delayed \ExpThree algorithm, called the Delay-Adaptive \ExpThree (\adaexp) algorithm, which only uses information available to the algorithm at the time of each decision, and achieves the optimal (up to a logarithmic factor) regret rate $\sqrt{\log(K)(KT + D)}$ (Section~\ref{sec:adaexp}). Compared to the results of \citet{ZiSe19}, our bounds are a logarithmic factor worse, which is due to the fact that our method is based on the simpler \ExpThree algorithm. In return, our analysis is much simpler.  
\item Combined with the implicit exploration technique of \citet{neu2015explore} in estimating the losses, we also derive a version of \ExpThree that achieves the \emph{first fully delay-adaptive high-probability regret bound} in the literature (Section~\ref{sec:highprob}).
While the latter bound also depends on the maximum delay $d^\star=\max_{t\in[T]} d_t$, this can be avoided by using the skipping technique \citep{ZiSe19}: for any set of time steps $R$ to be skipped, the resulting variants of the above \ExpThree algorithms achieve the optimal (up to a logarithmic factor) regret of order $\sqrt{KT\log(K)}+|R| + \sqrt{D_{\bar{R}}\log(K)}$ in expectation and with high probability, respectively (Section~\ref{sec:skipping}).
\item The performance of a learning algorithm can be significantly better than the minimax regret for nice problem instances. To take advantage of such situations in the delayed case, we develop a new version of the \adaexp algorithm, called Delay- and Data-Adaptive \ExpThree (\dedaexp), which is the first algorithm for the delayed setting whose expected regret scales with the actual (rather than the worst-case possible) losses, also improving our  bound for \adaexp (Section~\ref{sec:deda}). \dedaexp is based on a combination of our analysis technique introduced for \adaexp and the data- and delay-adaptive full-information algorithm of \citet{JoulaniGS16}. As a simple example of the resulting bounds, the algorithm achieves a regret of order $d^\star+\sqrt{\log(K) \left(d^\star L_{T,A^*} + \sum_{i=1}^K L_{T,i}\right)}$, where $L_{T,i}$ denotes the cumulative loss of action $i$ and $A^*$ denotes the optimal arm in $T$ time steps. This bound is essentially the same as the best data-dependent bound for \ExpThree (of order $\sqrt{\log(K)\sum_{i=1}^K L_{T,i}}$, as follows by \citealp{neu15}) and some extra delay term, where the effect of the delay depends only on the loss of the best arm but not of the other arms. 
\end{itemize}

On the technical side, the novelty in our analysis can be summarized as follows:
\begin{itemize}
\item We provide a direct reduction from the regret of the delayed-feedback bandit problem to that of a non-delayed (full-information) problem. As such, in contrast to previous work, our analysis does not need to modify the proof of the basic non-delayed exponential-weights algorithm; instead, we only need to bound the ``drift term'' arising from the reduction. 
Such a reduction has proved beneficial in the full-information setting \citep{JoulaniGS16}, but so far has not been found in the bandit setting \citep[Section 1]{ZiSe19}, partially because the delays change the order in which the losses are observed \citep[Appendix A]{thune19}. In addition to considerably simplifying the analysis of \adaexp (e.g., compared to what could be obtained for \ExpThree following the proof technique of \citet{ZiSe19}), this reduction is crucial for adopting the technique of \citet{JoulaniGS16} to the bandit setting and obtaining the data-adaptive bound for \dedaexp.
\item In addition, the drift term arising from this reduction is considerably easier to control than previous work, side-stepping the difficulties in the works of \citet{CeGeMa19,thune19,delayed_exp3_wrong} as mentioned above and in Footnote~\ref{fn:difficulty}.
\end{itemize}

\subsection{Notation}
We denote the set $\{1,2,\dots,n\}$ of the first $n$ natural numbers by $[n]$. The indicator of an event $\mathcal{E}$ is denoted by $\I{\mathcal{E}}$, taking the value $1$ if the event $\mathcal{E}$ happens and $0$ otherwise. For a sequence of functions, vectors, or scalars $a_s, a_{s+1}, \dots, a_t$, we use $a_{s:t}$ to denote the sum $\sum_{n=s}^{t} a_n$, with $a_{s:t} = 0$ if $s>t$.

\section{Problem formulation}
\label{sec:problem}

The multi-armed bandit problem is a sequential decision problem. Given a finite set of $K$ actions, denoted by $[K]=\{1,\ldots,K\}$, and the time horizon of the problem $T$, in every time step $t \in [T]$, the learner chooses an action $A_t \in [K]$ and suffers a loss $\ell_{t,A_t}$, where $\ell_t \in [0,1]^K$ is a loss vector such that $\ell_{t,i}$ is the loss associated with choosing action $i$ in time step $t$. We assume that the loss sequence $(\ell_t)_t$ is selected in advance and is not affected by the actions chosen by the learner (a.k.a.\ the oblivious setting).
As usual, we allow the learner to randomize, that is, at time step $t$ the learner determines a distribution $p_t$ in the $K-1$ dimensional probability simplex, and samples action $A_t$ from $p_t$ (conditionally independently of previous random choices, given $p_t$). With a slight abuse of terminology, sometimes we will refer to both $p_t$ and $A_t$ as the \emph{decision} of the learner at time $t$.

The learner's performance relative to any fixed action $A^\star$ is measured by the (expected) regret against $A^\star$, defined as
\[
R_T(A^\star) = \sum_{t=1}^T \EE{\ell_{t,A_t}} - \sum_{t=1}^T \ell_{t,A^*},
\]
and the learner aims to minimize its regret $R_T = \min_{A^* \in [K]} R_T(A^\star)$ against the best action in hindsight.

In the standard multi-armed bandit setting, after taking an action $A_t$, the learner immediately observes $\ell_{t,A_t}$, which can be used to improve its decisions in future time steps; however, the learner does not observe any loss $\ell_{t,i}$ for $i \neq A_t$.
In the \emph{delayed-feedback} setting we consider, the situation is somewhat different: after taking an action $A_t$ in time step $t$, the learner observes the loss $\ell_{t, A_t}$ only after a delay of $d_t$ time steps, after making a decision in time step $t+d_t$. 
This means that the decision $A_t$ in time step $t$ can only depend on the feedback which arrives before that time step, that is, on the losses $\{\ell_{s,A_s} : s+\dt_s <t\}$. Note that delay $\dt_t=0$ means that the corresponding feedback becomes available immediately after a decision is made. Without loss of generality%
\footnote{This is because the actions $A_1, \dots, A_T$ that determine the regret $R_T$ only depend on the feedback that arrives before time step $T$; any remaining feedback can thus be assumed to arrive at the end of time $T$ without affecting $R_T$.
}%
, we assume that all feedback arrives at the end of time step $T$, that is, $t+d_t \le T$ for all $t \in [T]$. We assume that the sequence of delays $(d_t)_t$ is selected before the process starts, obliviously to the actions of the learner. Note, however, that the losses and delays can be selected jointly with an arbitrary dependence among them.

\paragraph{Definitions.} The following definitions will be useful in analyzing the regret of delayed algorithms.
We use $I_{t,i} = \I{A_t = i}$ to indicate whether action $i \in [K]$ is played at time $t$.
The set of time steps with feedback missing when computing $p_t$ is denoted by $O_t = \{s \in [t-1]  :  s + \dt_s \ge t \}$ (with $O_{1} = \emptyset$). The number of missing feedbacks at time step $t$ is $\tau_t = | O_t | = \sum_{s=1}^{t-1} \I{s+\dt_s \ge t}$. The set of time steps where the feedback for time step $t$ is missing is denoted by $D_t = \{s: t \in O_s\}
= \{s: t < s \le t+\dt_t \}$. Note that the size of this set is $|D_t|=\dt_t$.
We denote the maximum delay by $d^\star=\max_{t \in [T]} d_t$ and the cumulative delay by $D=\sum_{t=1}^T d_t$. Note that
$D=\sum_{t=1}^T \tau_t$, since $D=\sum_{t=1}^T \sum_{s=1}^T \I{t \in O_s} = \sum_{s=1}^T \sum_{t=1}^T \I{t \in O_s} = \sum_{s=1}^T \tau_s$.

\paragraph{Loss estimates.} Standard bandit algorithms form some estimate $\hell_t$ of the loss vector $\ell_t$ when the feedback $\ell_{t,A_t}$ is received. A standard estimate is the importance-weighted estimator \citep{ACFS:2002}, defined as
\begin{align}
    \label{eq:ellhat1}
    \hell_{t,i} = \frac{\ell_{t,i} \I{A_t = i}}{p_{t,i}} = \frac{\ell_{t,A_t}I_{t,i}}{p_{t,i}}
    ,
    \qquad
    i \in [K]
\end{align}
Note that $\hell_{t,i}$ is computable when feedback $\ell_{t,A_t}$ has arrived as it does not depend on any other component of $\ell_t$.
Let  $\cH_{t} = \{(s, A_s, \hell_s) : s \in [t-1] \setminus O_t \}$ denote the history of the actual observations when computing $p_t$.
The estimator \eqref{eq:ellhat1} is unbiased as $\EE{\hell_t| p_t}= \ell_t$.
We consider learning algorithms whose decision $p_t$ depends on $\cH_t$, that is, $p_t$ is $\sigma(\cH_t)$-measurable (where $\sigma(\cH_t)$ denotes the $\sigma$-field generated by $\cH_t$).\footnote{Note, however, that this is a restriction, as the computation of $p_t$ could also depends on past decisions with missing feedback, that is, on $\{(p_s,A_s): s \in O_t\}$.}
Therefore, for the estimate $\hell_t$ in \eqref{eq:ellhat1} we have $\EE{\hell_t| \cH_t}= \ell_t$. 
In addition, we also consider loss estimates with the so-called implicit exploration (see, e.g., \citet{neu2015explore}):
\begin{align}
    \label{eq:ellhat-IX}
    \hell_{t,i} = \frac{\ell_{t,i} I_{t,i}}{p_{t,i} + \gamma_t}
    \,
    ,
\end{align}
where $\gamma_t \in \sigma(\cH_{t}), t \in [T]$ is a non-negative sequence of reals. The $\sigma(\cH_{t})$-measurability of $\gamma_t$ ensures that
$\EE{\hell_t| \cH_t} \le \ell_t$, thus encouraging exploration by reducing the observed expected loss.

\section{The Delay-Adaptive \ExpThree Algorithm}
\label{sec:adaexp}

Probably the simplest way to extend an algorithm designed for the non-delayed case to the delayed-feedback setting is to apply the same algorithm to the available losses. Such algorithms have been used extensively in the literature \citep[see, e.g.,][]{Joulani:2013,CeGeMa19}. In this section we analyze a similar extension of the \ExpThree algorithm, which we call the delay-adaptive \ExpThree (\adaexp) algorithm.
Wile at time $t$, in the non-delayed case, \ExpThree selects action $i$ with probability proportional to $e^{-\eta_t \hat{L}_{t,i}}$ for some step-size $\eta_t>0$, where $\hat{L}_t = \sum_{s=1}^{t-1} \hell_s$,\footnote{Note that, perhaps unusually, $\hat{L}_t$ is the sum of losses up to time step $t-1$, not $t$.} in the delayed case the set of available loss estimates is potentially smaller, and the decision is made based on $\hat{L}_t - \hat{\Delta}_t$, where $\hat{\Delta}_t= \sum_{s \in O_t} \hell_s$ is the sum of the missing loss estimates, which have not arrived, but would have arrived in the non-delayed setting.

Thus, at time $t$, \adaexp samples action $i$ with probability
\begin{align}
    p_{t,i} = \frac{\exp\left(-\eta_t (\hat{L}_{t,i} - \hat{\Delta}_{t,i}) \right)}{\sum_{j \in [K]} \exp\left(-\eta_t (\hat{L}_{t,j} - \hat{\Delta}_{t,i}) \right)}
    \, ,
    \label{eq:pt-def}
\end{align}
where $\eta_t>0$ is $\sigma(\cH_{t})$-measurable (i.e., $\eta_t$ may depend on any feedback information available at the beginning of time step $t$). 
\adaexp adapts to the delays by properly tuning the step-size $\eta_t$. In fact, this step-size tuning is the key contribution in the algorithm design.

The next result gives an upper bound on the expected regret of \adaexp.
\begin{theorem}
\label{thm:simple}
Suppose that the losses are in $[0,1]$, and $\eta_1, \eta_2, \dots, \eta_T$ is a positive, non-increasing sequence of step-sizes. Then, for all $A^\star \in [K]$, \adaexp satisfies
\begin{align}
    R_T(A^\star)
    \le 
    \EE{\eta_T^{-1}} \log(K) 
    + \sum_{t=1}^{T} \min\{1,\EE{\eta_t (\tau_t + K)}\}
    \, .
    \label{eq:final-regret-bound}
\end{align}
\end{theorem}
\begin{proof}
Let $p^\star$ be the probability distribution with all mass on $A^\star$. From the definition of regret and the loss-estimates, for any sequence of probability distributions $\tl{p}_{t+1}, t \in [T]$, we have
\begin{align}
    R_T(A^\star) & =
    \EE{\sum_{t=1}^{T} \hell_{t}^\top (p_t - p^\star)}
    =
    \EE{\sum_{t=1}^{T} \hell_{t}^\top (\tl{p}_{t+1} - p^\star)}
    +
    \EE{\sum_{t=1}^{T} \hell_{t}^\top (p_t - \tl{p}_{t+1})}
    \nonumber
    \\
    &
    =
    \EE{\sum_{t=1}^{T} \hell_{t}^\top (\tl{p}_{t+1} - p^\star)}
    +
    \EE{\sum_{t=1}^{T} \sum_{i=1}^{K} \hell_{t,i} p_{t,i} \left(1 - \frac{\tl{p}_{t+1,i}}{p_{t,i}} \right)}
    \, ,
    \label{eq:regret-decomposition}
\end{align}
where in the first equality we used the fact that $\ell_{t, A_t} = (\ell_{t,A_t} / p_{t, A_t} ) p_{t, A_t} = \hell_{t}^\top p_t$, and  
$\ell_{t, A^\star} = \ell_t^\top p^\star = \EE{\hell_{t}| \cH_t}^\top p^\star$. 
Now, let $\tl{p}_t$ be the (full-information) adaptive exponential-weights updates for the sequence of linear losses $\hell_t$, with non-increasing step-sizes $\eta_{t-1}$, that is, $\tl{p}_{1}$ is the uniform distribution over $[K]$, and for all $t \in [T]$,
\begin{align*}
    \tl{p}_{t+1, i} = \frac{\exp(-\eta_{t} \hat{L}_{t+1,i})}{\sum_{j=1}^{K} \exp(-\eta_t \hat{L}_{t+1,j})}~.
\end{align*}
Note that the index of $\eta$ is shifted by one, that is, $\tl{p}_{t+1}$ uses the same step-size $\eta_t$ that is used by $p_t$ (rather than $\eta_{t+1}$, which is used by $p_{t+1}$). This is not problematic: we only assume the imaginary iterate $\tl{p}_t$ uses slightly outdated information for tuning the step-size, and is still $\sigma(\cH_{t})$-measurable.

Thus, \eqref{eq:regret-decomposition} decomposes the regret into two terms: the first one is the ``cheating'' (or look-ahead) regret for the ideal (imaginary) exponential-weights iterate $\tl{p}_{t+1}$ (which depends on $\hell_t$ at time step $t$), while the second term is a ``drift'' term which measures the effect of using $p_t$ instead of $\tl{p}_{t+1}$.\footnote{
This type of decomposition involving the ``cheating regret'' is well-known in online learning (see, e.g., \citealp{joulani2020modular}). In fact, the non-delayed case can be thought of as the cheating case with a delay of $1$, where $\hell_t$ is not available at the time of computing $\tl{p}_t$, but is available at time $t+1$. In this sense, our decomposition follows naturally by collecting all such delayed losses only in the drift term; the cheating regret term can then be bounded in a black-box manner. This also has a further benefit in the bandit setting: had we used the standard regret (with $\tl{p}_t$ in place of $\tl{p}_{t+1}$), the products $\hell_t^\top \tl{p}_t$ that would show up in the standard regret decomposition would include a ratio $\tl{p}_{t,i} / p_{t,i}$ (due to the importance weight $p_{t,i}$ used in $\hell_{t,i}$), and bounding this ratio from above has been the source of much difficulty in previous work \citep{CeGeMa19,thune19,delayed_exp3_wrong}.
}

The first, cheating regret term can be directly bounded by Theorem~3 of \citet{joulani2020modular}%
\footnote{We have invoked Theorem 3 of \citet{joulani2020modular} with $p_t \equiv 0, t \in [T]$, $r_0 = (1/\eta_{0}) \sum_i p_i \log(p_i)$ and $r_t(p)=(1/\eta_t - 1/\eta_{t-1}) \sum_i p_i \log(p_i)$, $t \in [T]$, and dropped the Bregman-divergence terms due to the convexity of $r_t$.}%
 as
\begin{align}
    {\sum_{t=1}^{T} \hell_{t}^\top (\tl{p}_{t+1} - p^\star)} 
    \le 
    &\
    {\eta_{T}^{-1}} \log(K) 
    \,
    .
    \label{eq:full-info-regret}
\end{align}
To control the ``drift term'', that is, the second term on the right hand side of \eqref{eq:regret-decomposition},
we bound $\EE{\frac{\tl{p}_{t+1,i}}{p_{t,i}}}$ from below.
Observe that since the losses are non-negative, for all $t \in [T]$ and $i \in [K]$, $\hat{L}_{t,i}$, $\hat{\Delta}_{t,i}$ and $\eta_t$ are positive. Hence, we have
\begin{align*}
    \exp{\left(-\eta_t (\hat{L}_{t,i} - \hat{\Delta}_{t,i}) \right)}
    =
    \exp{\left(-\eta_t (\hat{L}_{t+1,i} - \hell_{t, i} - \hat{\Delta}_{t,i}) \right)}
    \ge
    \exp{(- \eta_t \hat{L}_{t+1,i})}
\end{align*}
for all $t \in [T]$ and $i \in [K]$, which implies
\begin{align}
    \frac{\tl{p}_{t+1,i}}{p_{t,i}}
    =
    &\ 
    \frac{\exp(-\eta_t \hat{L}_{t+1,i})}{\exp\left(-\eta_t (\hat{L}_{t,i} - \hat{\Delta}_{t,i}) \right)}
    \cdot
    \frac{\sum_{j \in [K]} \exp\left(-\eta_t (\hat{L}_{t,j} - \hat{\Delta}_{t,j}) \right)}{\sum_{j \in [K]} \exp(-\eta_t \hat{L}_{t+1,j})}
    \nonumber
    \\
    \ge
    &\ 
    \exp\left(-\eta_t \hat{\Delta}_{t,i} - \eta_t \hell_{t,i} \right)
    \ge 1 - \eta_t \hat{\Delta}_{t,i} - \eta_t \hell_{t,i}
    \, ,
    \label{eq:drift-ratio-lower-bound}
\end{align}
using in the last step the fact that $e^{x} \ge 1+x$ for all $x \in \Reals$. Thus, we can use \eqref{eq:drift-ratio-lower-bound} to upper-bound the second expectation on the right-hand-side of \eqref{eq:regret-decomposition} as
\begin{align}
    \EE{\sum_{i=1}^{K} \hell_{t,i}^\top p_{t,i} \left(1 - \frac{\tl{p}_{t+1,i}}{p_{t,i}} \right)} 
  & \le
    \EE{\sum_{i=1}^{K} \hell_{t,i} p_{t,i} \eta_t \hat{\Delta}_{t,i} + \eta_t \hell_{t,i}^2 p_{t,i}
    \label{eq:basic-bound-on-delay-penalty}}
    \\
    &=
    \EE{\sum_{i=1}^{K} \ell_{t,i} I_{t,i} \eta_t \sum_{s \in O_t} \frac{\ell_{s,i}I_{s,i}}{p_{s,i}}} 
    + 
    \EE{ \sum_{i=1}^{K} \eta_t \ell_{t,i}^2 I_{t,i} / p_{t,i}}
    \nonumber
    \\
    &= 
    \EE{ \sum_{i=1}^{K} \sum_{s \in O_t} \ell_{t,i} \ell_{s,i} \EE{\eta_t \frac{I_{t,i} I_{s,i}}{p_{s,i}}  \bigg| \cH_t}}
    \nonumber
    + 
    \EE{ \sum_{i=1}^{K}  \ell_{t,i}^2 \EE{ \frac{\eta_t I_{t,i}}{ p_{t,i} } \Big| \cH_t} }
    \nonumber
    \\
    &=
    \EE{ \sum_{i=1}^{K} \sum_{s \in O_t} \ell_{t,i} \eta_t \ell_{s,i} p_{t,i}}
    + 
    \EE{\sum_{i=1}^{K} \eta_t \ell_{t,i}^2}
    \nonumber
    \\
    &
    \le
    \EE{ \sum_{i=1}^{K} p_{t,i} \eta_t \tau_t}
    + 
    \EE{\eta_t K}
    =
    \EE{\eta_t (\tau_t + K)}
    \, ,
    \label{eq:delay-penalty-bound}
\end{align}
where in the second line we have used the definitions of $\hat{\Delta}_{t,i}$ and $\hell_{t,i}$ for $t=1,2,\dots,T$, and in the third line we have used the tower rule. 
The fourth step follows since $p_{s,i}$ and $p_{t,i}$ are determined by the feedback that is received by time $t$, and since the feedback for $I_{s,i}$ is missing, $I_{s,i}$ and $I_{t,i}$ are conditionally independent given $\cH_t$ with distributions $p_{s,i}$ and $p_{t,i}$, respectively. 
The last inequality uses the assumption that the losses are upper-bounded by $1$.

Combining \eqref{eq:delay-penalty-bound} with
\[
\sum_{i=1}^{K} \hell_{t,i}^\top p_{t,i} \left(1 - \frac{\tl{p}_{t+1,i}}{p_{t,i}} \right)
\le \sum_{i=1}^K \ell_{t,i}\I{A_t = i} = \ell_{t,A_t}
\le 1,
\]
summing up for all $t$ and putting back the resulting bound into \eqref{eq:regret-decomposition}, together with \eqref{eq:full-info-regret}, we get the regret bound \eqref{eq:final-regret-bound}.
\end{proof}

Using the non-increasing step-size sequence  $\eta_t=\sqrt{\frac{\log(K)}{tK + \sum_{s=1}^t \tau_s}}$, we obtain the first fully delay-adaptive bound for the \ExpThree-family of algorithms:
\begin{corollary} 
\label{cor:simple}
With $\eta_t=\sqrt{\frac{\log(K)}{tK + \sum_{s=1}^t \tau_s}}$, the regret of the \adaexp algorithm can be bounded as
\begin{align*}
    R_T
    \le 3\sqrt{\log(K)\left(TK + D \right)}~.
\end{align*}
\end{corollary}
\begin{proof}
The result is a direct corollary of Theorem~\ref{thm:simple}: 
We can bound the
second term on the right hand side of \eqref{eq:final-regret-bound} by
the standard inequality that for 
any $a_t >0$, $\sum_{t=1}^T a_t/\sqrt{\sum_{s=1}^t a_s} \le 2 \sqrt{\sum_{t=1}^T a_t}$ (see, e.g., \citep[Lemma~4]{mcmahan2017survey}). Applying this inequality for $a_t=K + \tau_t$, we obtain
\begin{align*}
    \sum_{t=1}^{T} \left(\EE{\ell_{t,A_t}} - \ell_{t, A^\star}\right)
    \le  3\sqrt{\log(K)\left(TK + \sum_{t=1}^T \tau_t\right)}
    \,
    .
\end{align*}
The statement of the corollary then follows by the fact that $D=\sum_{t=1}^T \tau_t$.
\end{proof}

\section{High-probability bounds}
\label{sec:highprob}
In this section, we show that using the loss estimate with implicit exploration (equation~\ref{eq:ellhat-IX}) enables us to prove a regret bound that hold with high probability instead of holding only in expectation.

\begin{theorem}
\label{thm:simple-high-prob}
Suppose \adaexp is run with a non-increasing step-size sequence $(\eta_t)$ using the loss estimate \eqref{eq:ellhat-IX} with $\gamma_t = \eta_t$. Let $\delta \in (0,1)$ and $A^\star \in [K]$ be an arbitrary action. Then, with probability at least $1- \delta$, the regret of the algorithm against $A^\star$ can be bounded as
\begin{align}
    \sum_{t=1}^{T} \left(\ell_{t,A_t} - \ell_{t, A^\star}\right)
    \le
    \frac{3 \log(K)}
    {2 \eta_{T}}
    +
    \sum_{t=1}^{T}
    \eta_t 
    \left(
    \tau_{t}
    + 
    2
    K
    \right)
    + 
    \left(
    \frac
    {\eta_T^{-1} + d^\star  +2}
    {2}
    \right)
    \log(2/\delta)
    \,
    .
    \label{eq:final-regret-bound-high-prob}
\end{align}
\end{theorem}

This implies the following corollary, which is proved exactly as Corollary~\ref{cor:simple}.
\begin{corollary} 
\label{cor:simple-high-prob}
With $\eta_t=\frac{1}{2} \sqrt{\frac{3 \log(K)}{2 tK + \sum_{s=1}^t \tau_s}}$, under the conditions of Theorem~\ref{thm:simple-high-prob},
\begin{align*}
    \sum_{t=1}^{T} \left(\ell_{t,A_t} - \ell_{t, A^\star}\right)
    \le
    2
    \sqrt{
    3
    \log(K)
    \left( 2 K T + D 
    \right)
    }
    + 
    \left(
    2
    \sqrt{\frac{2 TK + D}{3\log(K)}}
    + d^\star  +2
    \right)
    \frac{\log(2/\delta)}{2}
    \,
    .
\end{align*}
\end{corollary}
Note that the dependence on $\delta$ can be improved if the step size $\eta_t$ also depends on $\delta$. In the corollary we chose to tune the algorithm to be oblivious to the error parameter $\delta$.

To prove the theorem, we use the following lemma of \citet{neu2015explore}.%
\footnote{
\citet{neu2015explore} states the lemma for ``a fixed sequence'' of $\gamma_t$, but this is not used anywhere in their proof; their proof goes through without change, as long as $\gamma_t$ is determined by the history $\mathcal{H}_{t}$.
}
\begin{lemma}
[Lemma 1 of \citealp{neu2015explore}]
\label{lem:high-prob-base}
For $t \in [T], i \in [K]$, let $\gamma_t, \alpha_{t,i}$ be non-negative
$\cH_{t}$-measurable random variables satisfying $\alpha_{t,i} \le 2 \gamma_t$, and let $\hell_{t,i}$ be given by \eqref{eq:ellhat-IX}. Then, with probability at least $1 - \delta$,
\begin{align*}
\sum_{t=1}^{T}
\sum_{i=1}^{K}
\alpha_{t,i} 
\left(
\hell_{t,i} - \ell_{t,i}
\right)
\le
\log( 1/ \delta)
\,
.
\end{align*}
\end{lemma}

Next, we prove the theorem using ideas from the proofs of Theorem~\ref{thm:simple} and Theorem 1 of \citet{neu2015explore}.
\begin{proof}[Proof of Theorem~\ref{thm:simple-high-prob}]
We apply similar ideas as in the proof of Theorem~\ref{thm:simple}, but instead of taking expectations to control the loss estimate terms, we use Lemma~\ref{lem:high-prob-base}
to replace the loss estimates by the true losses at the expense of an additive logarithmic penalty, with high probability.

To start, we define $\tl{p}_{t}$ as in the proof of Theorem~\ref{thm:simple}, which implies that we have \eqref{eq:full-info-regret} and \eqref{eq:drift-ratio-lower-bound}. Therefore,
on the one hand,
\begin{align}
    \sum_{t=1}^{T} \hell_{t}^\top (p_t - \tl{p}_{t+1}) &= \sum_{t=1}^{T} \sum_{i=1}^{K} \hell_{t,i}^\top p_{t,i} \left(1 - \frac{\tl{p}_{t+1,i}}{p_{t,i}} \right)
  \nonumber
  \\
  &
  \le
    \sum_{t=1}^{T}
    \sum_{i=1}^{K} \hell_{t,i} p_{t,i} \eta_t \hat{\Delta}_{t,i} +
    \sum_{t=1}^{T}
    \eta_t \hell_{t,i}^2 p_{t,i}
    \nonumber
    \\
    &
    \le
    \sum_{t=1}^{T}
    \sum_{i=1}^{K} \ell_{t,i} I_{t,i} \eta_t \sum_{s \in O_t} \hell_{s,i}
    + 
    \sum_{t=1}^{T}
    \sum_{i=1}^{K} \eta_t \hell_{t,i}
    \nonumber
    \\
    &= 
    \sum_{s=1}^{T}
    \sum_{i=1}^{K}
    \hell_{s,i}
    \left(
    \sum_{t: s \in O_t}
    \ell_{t,i} I_{t,i} \eta_t 
    \right)
    + 
    \sum_{t=1}^{T}
    \sum_{i=1}^{K} \eta_t \hell_{t,i}
    \nonumber
    \\
    &\le
    \sum_{s=1}^{T}
    \sum_{i=1}^{K}
    \hell_{s,i}
    \left(\sum_{t:s\in O_t} \eta_t I_{t,i}\right)
    + 
    \sum_{t=1}^{T}
    \sum_{i=1}^{K} \eta_t \hell_{t,i}
    \, ,
    \label{eq:delay-penalty-bound-high-prob}
\end{align}
where the second line follows from \eqref{eq:drift-ratio-lower-bound}, the third line follows by the fact that $\hell_{t,i} p_{t,i} \le \ell_{t,i} I_{t,i} \le 1$, and the last line follows since $\ell_{t,i} \in [0,1]$.

On the other hand, we can derive a deterministic counterpart of \eqref{eq:regret-decomposition} as follows:
\begin{align}
     \sum_{t=1}^{T} \left(\ell_{t,A_t} - \ell_{t, A^\star}\right)
     & =
     \sum_{t=1}^{T} \hell_{t}^\top (\tl{p}_{t+1} - p^\star)
    +
     \sum_{t=1}^{T} \hell_{t}^\top (p_t - \tl{p}_{t+1})
    +
     \epsilon^\star_{1:T}
     +
     \sum_{t=1}^{T}
     \left(
     \ell_{t,A_t} - \hell_{t}^{\top} p_t
     \right)
     \, ,
    \label{eq:regret-decomposition-high-prob}
\end{align}
where $\epsilon^\star_t = \hell_{t,A^\star} - \ell_{t,A^\star}$.
Following the proof of Theorem 1 of \citet{neu2015explore}, it is easy to show that
\begin{align*}
    \sum_{t=1}^{T}
    \left(
    \ell_{t,A_t} - \hell_{t}^{\top} p_t
    \right)
    = \sum_{t=1}^{T} \gamma_t \sum_{i=1}^{K} \hell_{t,i}
    \, .
\end{align*}
Then, combining with \eqref{eq:delay-penalty-bound-high-prob}, we have
\begin{align*}
    \sum_{t=1}^{T} \hell_{t}^\top (p_t - \tl{p}_{t+1})
    +
    \sum_{t=1}^{T}
    \left(
    \ell_{t,A_t} - \hell_{t}^{\top} p_t
    \right)
    &
    \le 
    \sum_{t=1}^{T}
    \sum_{i=1}^{K}
    \left(\eta_t + \gamma_t + \sum_{s:t \in O_s} \eta_s I_{s,i}\right)
    \hell_{t,i}
    \\
    &
    =
    \sum_{t=1}^{T}
    \sum_{i=1}^{K}
    \left(\eta_t + \gamma_t + \sum_{s:t \in O_s} \eta_s I_{s,i}\right)
    \ell_{t,i}
    +
    \frac{d^\star + 2}
    {2}
    \hat{\epsilon}_{1:T}
    \, ,
\end{align*}
where $\hat{\epsilon}_t=
\sum_{i=1}^{K} 
2
\frac{
\eta_t + \gamma_t + \sum_{s:t \in O_s} \eta_s I_{s,i}}
{d^\star + 2}
(\hell_{t,i} - \ell_{t,i})$. Note that in the latter definition, the coefficient of $(\hell_{t,i} - \ell_{t,i})$ is bounded by $2\gamma_t$ since $\eta_t=\gamma_t$, $\eta_t \ge \eta_s$ for all $s$ such that $t \in O_s$ (since $s > t$ in this case), and $d^\star \ge |O_s|$. 
Hence, with probability at least $1-\delta'$,  $\hat{\epsilon}_{1:T} \le \log(1/\delta')$  by Lemma~\ref{lem:high-prob-base} for any $\delta' \in (0,1)$.
Furthermore, 
\[
\sum_{t=1}^{T}
    \sum_{i=1}^{K} \sum_{s: t\in O_s} \eta_s I_{s,i} = 
    \sum_{s=1}^T \sum_{t: t\in O_s} \eta_s = \sum_{s=1}^T \eta_s \tau_s~.
\]
Therefore, using that $\ell_{t,i} \in [0,1]$ and $\eta_t=\gamma_t$, the first term on the right hand side above can be bounded as
\[
\sum_{t=1}^{T}
    \sum_{i=1}^{K}
    \left(\eta_t + \gamma_t + \sum_{s:t \in O_s} \eta_s I_{s,i}\right)
    \ell_{t,i}
\le \sum_{t=1}^{T} 2K \eta_t + \sum_{s=1}^T \eta_s \tau_s = \sum_{t=1}^{T} \eta_t(\tau_t + 2K)\,.
\]

Putting these back into \eqref{eq:regret-decomposition-high-prob}, combining with \eqref{eq:full-info-regret}, and letting $\gamma_t = \eta_t$, for any $\delta' \in [0,1]$ we have
\begin{align*}
    \sum_{t=1}^{T} \left(\ell_{t,A_t} - \ell_{t, A^\star}\right)
    & \le
    \frac{\log(K)}{\eta_{T}} + \sum_{t=1}^{T} \eta_t (\tau_t + 2K) +
    \frac{d^\star + 2}{2} \hat{\epsilon}_{1:T} + \epsilon^\star_{1:T}
    \\
    & \le 
    \frac{\log(K)}{\eta_{T}} + \sum_{t=1}^{T} \eta_t (\tau_t + 2K) 
    + \frac{\log(K/\delta')}{2 \eta_T}
    + \frac{d^\star+2}{2}\log(1/\delta')
    \\
    & =
    \frac{3\log(K)}{2\eta_{T}} + \sum_{t=1}^{T} \eta_t(\tau_t + 2K)
    + \frac{\eta_T^{-1} + d^\star + 2}{2} \log(1/\delta')
    \,
    ,
\end{align*}
with probability at least $1-2\delta'$, where we also used that
\begin{align*}
\epsilon^\star_{1:T} 
&
\le
\frac{1}{2\gamma_T}
\sum_{t=1}^{T}
2 \gamma_t
\left(
\hat{\ell}_{t,A^*} - \ell_{t,A^\star}
\right) 
= 
\frac{1}{2\gamma_T}
\sum_{t=1}^{T}
\sum_{i=1}^{K}
2 \gamma_t
\I{A^* = i}
\left(
\hat{\ell}_{t,i} - \ell_{t,i}
\right) 
\le
\frac{1}{2\gamma_T}
\log(K/\delta')
\end{align*}
with probability at least $1-\delta'$ simultaneously for all $A^\star$ by  Lemma~\ref{lem:high-prob-base} and the union bound. 

Letting $\delta' = \delta/2$ and using the assumption that $\ell_{t,i} \le 1$ completes the proof.
\end{proof}

\section{Skipping time steps}
\label{sec:skipping}

Looking at the form of the bound in Theorem~\ref{thm:simple}, one can observe that the terms in the second summation are a minimum of $1$ and $\eta_t(\tau_t+K)$, where the latter comes from bounding the drift terms. As such, whenever $1$ is smaller than $\eta_t(\tau_t + K)$, our analysis in \eqref{eq:delay-penalty-bound} is too pessimistic.
The effect of this could be avoided by keeping the minimum term when we define the step size $\eta_t$, but this is not straightforward if we want to keep the simple sum structure of $1/\eta_t^2$, which is used in the proof of Corollary~\ref{cor:simple}. A simpler approach is to ensure that $\tau_t$ never becomes too large (compared to $\eta_t$), which can be done by limiting individual delays $d_t$ by pretending that their corresponding loss arrives (and has value 0) when $d_t$ is too large, and bound the regret in the corresponding time step separately by 1.
This essentially means that the algorithm eventually skips time steps with excessive delays. This also addresses another problem: namely that the cumulative delay $D$ can be dominated by a few large delay values, which--intuitively--should not cause such a large penalty in the regret.

The idea of skipping, originally coined by \citet{thune19}, was perfected by \citet{ZiSe19}, who provided a proper way to skip some time steps and tune the step size accordingly. In what follows, we adopt their way of tuning the step size, and closely follow their analysis (which they do in the context of their more complicated follow-the-regularized-leader algorithm).

Next we describe how the tuning method of \citet{ZiSe19} can be adapted to our \adaexp algorithm:
The advanced tuning method of skipping procedure works as follows: For every time step $t$ and time steps $s \le t$, we keep a binary indicator $a_s^t \in \{0,1\}$ such that for any round $t$ we include the loss from time $s$ to the set of missing losses if $a_s^t=1$, and not if $a_s^t$. That is, the number of \emph{counted} missing losses is
\[
\dtau_t = \sum_{s=1}^t a_s^t \I{s+\dt_s \ge t}~.
\]
Let $\dD_t = \sum_{s=1}^t \dtau_t$ denote the cumulative number of counted missing feedbacks. $a_s^t$ is originally set to $1$ for all $t \ge s$, but if $s \in O_t$ and $\min\{\dt_s,t-s\} > \sqrt{\dD_t/\log(K)}$, we set $a_s^{t'}=0$ for all $t'>t$ (by Lemma~7 of \citealp{ZiSe19}, this happens for at most one $s$ value in any time step $t$).

Tuning the step size of \adaexp with $\dtau$ instead of $\tau$, the regret of the \dedaexp algorithm can be bounded as follows:

\begin{theorem}
(i) Bound in expectation: The expected regret of the \adaexp algorithm with loss estimates \eqref{eq:ellhat1} and step sizes $\eta_t=\sqrt{\frac{\log(K)}{tK + \sum_{s=1}^t \dtau_t}}$ can be bounded as
\begin{align*}
    \sum_{t=1}^{T} \left(\EE{\ell_{t,A_t}} - \ell_{t, A^\star}\right)
    \le 3\sqrt{TK\log(K)} + 10 \max\left\{2 \log K, 
    \min_{R \subset [T]} \left(|R| + \sqrt{D_{\bar{R}}
    \log(K)}\right)\right\},
\end{align*}
where for any $R \subset [T]$, $\bar{R}=[T]\setminus R$, and $D_{\bar{R}}=\sum_{t \in \bar{R}} \dt_t$.

(ii) High-probability bound: Let $\delta \in (0,1)$. The regret of the \adaexp algorithm with loss estimates \eqref{eq:ellhat-IX} and step sizes $\eta_t=\gamma_t=\frac{1}{2} \sqrt{\frac{3 \log(K)}{2 tK + \sum_{s=1}^t \tau_s}}$ can be bounded, with probability at least $1-\delta$, as
\begin{align*}
\sum_{t=1}^{T} \left(\ell_{t,A_t} - \ell_{t, A^\star}\right)
& \le \left(c_1 + c_2 \frac{\log(2/\delta)}{\log(K)}\right)
\sqrt{KT\log(K)} \\
& \qquad +
\left(c_3 + c_4 \frac{\log(2/\delta)}{\log(K)}\right)
\max\left\{2 \log K, 
    \min_{R \subset [T]} \left(|R| + \sqrt{D_{\bar{R}}
    \log(K)}\right)\right\},
\end{align*}
where $c_1=2\sqrt{6}$, $c_2=\sqrt{2/3}$, $c_3=4(\sqrt{3}+1)$, and $c_4=1+2/\sqrt{3}$.
\end{theorem}

The meaning of the theorem is that the regret of the algorithm is essentially of the same order as if a set of time steps $R$ was to be skipped, and the algorithm was only run on its complement $\bar{R}$. Note that while our original high-probability regret bound (cf. Corollary~\ref{cor:simple-high-prob}) depended on the maximum delay $d^\star$, this dependence is eliminated from the high-probability bound of the theorem, as the maximum delay is effectively bounded by $\sqrt{\tilde{D}_T/\log(K)}$.

\begin{proof}
First define the \emph{effective} delay for time step $s$ as $\tdt_s=\sum_{t=s+1}^{s+d_s} a_s^t$; that is, if time step $s$ is not ``skipped'' (i.e., $a_s^t=1$ for all $t \in [T]$), $\tdt_s=\dt_s$, and $\tdt_s=t-s$ if $s$ is skipped at the end of time step $t$, that is, $a_s^{t+1}=0$ and $a_s^t=1$.
Let $S=\{t\in[T]: a_t^T=0\}$ denote the set of skipped time steps, and define a new loss sequence $(\tell_t)$ such that the loss is zeroed out if the corresponding time step is ever skipped by the algorithm, that is, $\tell_t=\ell_t$ if $t \not\in S$ and $\tell_t=0$ if $t \in S$. Note that this loss sequence can be constructed deterministically from the loss sequence $(\ell_t)$ and the delay sequence $(\dt_t)$. 

It is easy to see that the $\tau_t$-dependent tuning of \adaexp, considered in this theorem, results in exactly the same sequence of predictions $(p_t)$ as the original \adaexp algorithm for the losses $(\tell_t)$ and delays $(\tdt_t)$. 

To prove the upper bound on the expected regret, we start by applying  Corollary~\ref{cor:simple} for the latter case:
\begin{align}
    \sum_{t=1}^{T} \left(\EE{\ell_{t,A_t}} - \ell_{t, A^\star}\right)
    & \le
    |S| + \sum_{t=1}^{T} \left(\EE{\tell_{t,A_t}} - \tell_{t, A^\star}\right) 
    \nonumber \\ 
    & 
    \le
    |S| +  3\sqrt{\log(K)\left(TK + \sum_{t=1}^T \tdt_t\right)}~.
    \label{eq:skipped}
\end{align}

To finish, we need to bound $|S|$, and relate the resulting bound to the bound in the theorem for an arbitrary $R \subset [T]$. To do so, we recycle a few results from \citet{ZiSe19}:
In their Lemma~5, they show that $|S| \le 2 \sqrt{\log(K)\sum_{t=1}^T \tdt_t}$.
Furthermore, in the proof of their Theorem~2, they show that if $\sum_{t=1}^T \tdt_t \ge 16 \log(K)$, then
\[
\sqrt{\sum_{t=1}^T \tdt_t \log(K)} \le 2 \min_{R \subset [T]} \left(|R| + \sqrt{\sum_{t \in \bar{R}} \dt_t \log(K)}\right).
\]
Combining these results with \eqref{eq:skipped} (and using $\sqrt{a+b} \le \sqrt{a}+\sqrt{b}$ for any $a,b>0$), we obtain that the expected regret can be bounded as
\begin{align*}
   \lefteqn{ \sum_{t=1}^{T} \left(\EE{\ell_{t,A_t}} - \ell_{t, A^\star}\right)}\\
    & \le 3\sqrt{TK \log(K)} + 5 \sqrt{\log(K)\sum_{t=1}^T \tdt_t} \\
    & \le 3\sqrt{TK \log(K)} + 10 \max\left\{2 \log K, 
    \min_{R \subset [T]} \left(|R| + \sqrt{\sum_{t \in \bar{R}} \dt_t \log(K)}\right)\right\},
\end{align*}
proving the bound on the expected regret.

To get the high-probability bound, we use Corollary~\ref{cor:simple-high-prob} to get a high-probability version of \eqref{eq:skipped}:
With $\eta_t=\gamma_t=\frac{1}{2} \sqrt{\frac{3 \log(K)}{2 tK + \sum_{s=1}^t \tau_s}}$, and defining $\tdt^\star=\max_{t\in[T]} \tdt_t$, we obtain that with probability at least $1-\delta$,
\begin{align}
    \lefteqn{\sum_{t=1}^{T} \left(\ell_{t,A_t} - \ell_{t, A^\star}\right)} \nonumber \\
    & \le
    |S| + 
    2
    \sqrt{
    3
    \log(K)
    \left( 2 K T + \sum_{t=1}^T \tdt_t 
    \right)
    }
    + 
    \left(
    2
    \sqrt{\frac{2 TK + \sum_{t=1}^T \tdt_t}{3\log(K)}}
    + \tdt^\star  +2
    \right)
    \frac{\log(2/\delta)}{2}
    \,
    .
\end{align}
By construction, $\tdt^\star \le \sqrt{\tilde{D}_T/\log(K)} = \sqrt{ \sum_{t=1}^T \tdt_t/\log(K)}$. Using this and the same steps as for the regret bound in expectation proves the high-probability bound of the theorem.
\end{proof}

\section{Adapting to delay and data at the same time}
\label{sec:deda}
In this section, 
we consider a different step-size sequence 
that yields AdaGrad-style bounds. Recall that from
\eqref{eq:basic-bound-on-delay-penalty} in the proof of Theorem~\ref{thm:simple}, we have
\begin{align*}
    \sum_{t=1}^{T}
    \sum_{i=1}^{K}
    \hell_{t,i}^\top p_{t,i} 
    \left(
    1 - \frac{\tl{p}_{t+1,i}}{p_{t,i}}
    \right)
    \le
    \sum_{t=1}^{T}
    \eta_t
    \sum_{i=1}^{K}
    \lfwd_{t,i},
\end{align*}
where $\lfwd_{t,i} = \hell_{t,i} p_{t,i} \hat{\Delta}_{t,i}
    + \hell_{t,i}^2 p_{t,i}$.
Therefore, ideally, we want to set the step-size $\eta_t$ as
\begin{align}
    \eta_t
    =
    \sqrt{\frac{\log(K)}{\sum_{i=1}^{K} \lfwd_{1:t, i}}}
    \label{eq:adaptive-step-mixed-delay-and-data-IDEAL}
    \,
    ,
\end{align}
to optimize the regret bound and obtain a data-adaptive bound of the form
\begin{align}
    R_T
    \le 
    3
    \EE{
    \sqrt{
    \log(K)
    \sum_{i=1}^{K}
    \lfwd_{1:T,i}
    }
    }
    \, .
    \label{eq:delay-data-adaptive-bound-aaai-IDEAL}
\end{align}
However, we do not have access to the missing observations $\hell_s, s \in O_t \cup \{t\}$, when calculating $\eta_t$. Therefore, we approximate the step-size of \eqref{eq:adaptive-step-mixed-delay-and-data-IDEAL} with another sequence that can be computed efficiently at each time step. Algorithm~\ref{alg:delay-and-data-adaexp} provides the details of this approximation.

\begin{algorithm}[t]
    \KwIn{Number of actions $K$.}
    
    \textbf{Initialization:}

    \begin{fleqn}
    \begin{align*}
        &\
        z_{1,i}  \gets 0,  m_{1,i} \gets 0 \text{ for all } i \in [K]
        .
        \\
        &\
        d^\star_0  \gets 0,  \Lbck_1 \gets 0
        .
        \\[-20pt]
        \end{align*}
    \end{fleqn}
    
    \For{$t = 1, 2, \dots, T$}{
        $d^\star_{t}
            \gets
            \max\{d_{t}, d^\star_{t-1}\}
        $
        .
        
        $\frac{1}{\gamma_t} 
        = \frac{1}{\eta_t}
        \gets
        \frac
        {
        4 (d^\star_t)^2
        + 6 d^\star_t
        + 2
        }
        {\log(K)}
        +
        \sqrt{
        \frac
        {
         \Lbck_t
        }
        {\log(K)}
        }
        $
        .
        
        $p_{t, i} 
            \gets 
            \frac
            {\exp(- \eta_t z_{t,i})}
            {\sum_{j=1}^{K} 
            \exp(- \eta_t z_{t,j})
            }
        $
        .
        
        Play action $A_t \in [K]$ selected randomly according to distribution $p_t$.
        
        Store $\gamma_t, m_{t}$, $z_{t}$, $p_{t}$, and $A_t$ in memory.
        
        \For{$s: s+d_s = t$}{
        Retrieve $A_s, p_s, m_s, z_s$ and $\gamma_s$ from memory.
        
        Let $\hell_{s,i} = \frac
        {\ell_{s,A_s} I_{s,i}}
        {p_{s,i}+\gamma_s}
        $
        for all
        $i \in [K]$.
        }

        \begin{fleqn}
        \begin{align*}
            &\
            z_{t+1, i}
            \gets 
            z_{t,i} + \sum_{s: s+d_s = t} \hell_{s,i}, 
            \text{ for all }
            i \in [K]
            .
            & & & & & &
            & & & & & &
            & & & & & &
            \\
            &\
            m_{t+1, i}
            \gets
            m_{t,i} + \sum_{s: s+d_s = t} \hell_{s,i} p_{s,i},
            \text{ for all }
            i \in [K]
            .
            & & & & & &
            & & & & & &
            & & & & & &
            \\[-20pt]
        \end{align*}
        \end{fleqn}
        
        \begin{fleqn}
        \begin{align*}
            \Lbck_{t+1}
            \gets 
            \Lbck_{t} 
            +
            \sum_{i=1}^{K}
            \sum_{s: s+d_s = t}
            \left(
            \hell_{s,i} (m_{t+1,i} - m_{s,i})
            +
            \hell_{s,i} p_{s,i}
            \left(
            z_{t+1,i} - z_{s,i}
            \right)
            \right)
            .
            \\[-20pt]
        \end{align*}
        \end{fleqn}
    }
    \caption{Delay and Data Adaptive \ExpThree (\dedaexp).}
    \label{alg:delay-and-data-adaexp}
\end{algorithm}

The algorithm keeps track of the largest delay ($d^\star_t$) and, similarly to \citet{thune19}, to do so it needs to observe the delay for action $A_t$ in advance. This can be avoided if the algorithm has access to an a priori upper bound $d^B$ on the maximum delay: setting $d^\star_0$ to this upper bound results in $d^\star_t=d^B$ for all $t \in [T]$.
As usual in AdaGrad-style algorithms, the algorithm maintains additional vectors $m_t$, $z_t$ and $\Lbck_t$ to compute the step size. In addition, the algorithm uses a memory to store values of $m_s$, $z_s$, and $p_s$ for past steps $s$ with missing feedback. In particular, after coming up with the action distribution $p_s$ for time step $s$, we store the current values of $m_s$ and $z_s$, as well as the action distribution $p_{s}$ and the action taken, $A_s$. When the feedback for time step $s$ arrives at the end of time step $t = s + d_s$, we retrieve these values, and use them to compute $z_{t+1}$, $m_{t+1}$, and $\Lbck_{t+1}$.

\paragraph{Interpretation of $m_t$, $z_t$, and $\Lbck_t$.} It is easy to verify that for all $t \in [T+1]$ and $i \in [K]$,
\begin{align*}
    z_{t,i}
    =
    \sum_{j: j+d_j < t} \hell_{j,i}
    \,
    ,
    \qquad
    \text{and,}
    \qquad
    m_{t,i}
    =
    \sum_{j: j+d_j < t}
    \hell_{j,i} p_{j,i}
    \,
    .
\end{align*}
In addition, if we define
\begin{align*}
    \lbck_{s,i}
    =
    \sum_{j: s \le j+d_j \le s+d_s}
    \hell_{s,i}
    \hell_{j,i} 
    (p_{j,i}+p_{s,i})
    \,
    ,
\end{align*}
then it is easy to see that
$\lbck_{s,i} = \hell_{s,i}
\left(
  m_{s+d_s+1, i} - m_{s,i}
\right)
+ \hell_{s,i} p_{s,i} 
\left(
z_{s+d_s+1, i} - z_{s,i}
\right)
$. Therefore,
$\Lbck_{t} = \sum_{i=1}^{K} \sum_{s: s+d_s < t} \lbck_{s, i}$, and the algorithm is using the step-size schedule
\begin{align}
    \eta_t^{-1}
    =
    \frac
    {
    4 (d^\star_t)^2
    + 6 d^\star_t
    + 2
    }
    {\log(K)}
    +
    \sqrt{
    \frac
    {\sum_{i=1}^{K}
    \sum_{s:s+d_s < t} \lbck_{s, i}
    }
    {\log(K)}
    }
    \,
    .
    \label{eq:adaptive-step-mixed-delay-and-data}
\end{align}
In Section~\ref{sec:deda-analysis}, we show that this step-size efficiently approximates the ideal step-size  \eqref{eq:adaptive-step-mixed-delay-and-data-IDEAL}, in the sense that it results only in a lower-order penalty in the regret compared to the regret bound of \eqref{eq:delay-data-adaptive-bound-aaai-IDEAL}.

\paragraph{Memory requirement.} The implementation of the memory required by Algorithm~\ref{alg:delay-and-data-adaexp} is application-dependent, and is very straightforward in specific cases. For example, in a distributed setting, the worker node that is responsible for decision $s$ will simply receive $m_s, z_s, p_{s}$ before making a decision, and return them along with the feedback after receiving it. In any case, the memory can be explicitly implemented using a hash table with an amortized computation cost of $\Theta(1)$ per storage and retrieval, and a total storage of $\mathcal{O}\left( d^\star_T \right)$.

\subsection{Analysis}
\label{sec:deda-analysis}
We start the analysis by showing that the backward loss estimates $\ell^{\textnormal{bck}}_{t,i}$ are not far away from the forward losses $\ell^{\textnormal{fwd}}_{t,i}$. Hence, the step-size sequence
given by
\eqref{eq:adaptive-step-mixed-delay-and-data} will result in a regret not far away from what could be achieved by 
the ideal step-size \eqref{eq:adaptive-step-mixed-delay-and-data-IDEAL}.
This is captured by the following lemma, proved in Section~\ref{sec:proof-of-lemma}:
\begin{lemma}
[Step-size control]
\label{lem:step-size-control}
For all $t \in [T]$,
\begin{align}
    \sum_{s: s+d_s < t}
    \lbck_{s,i}
    \le
    &
    2
    \sum_{j = 1}^{t}
    \sum_{s \in O_j \cup \{j\} \cup D_j}
    \hell_{s,i}
    \hell_{j,i}
    p_{j,i}
    \,
    ,
    \label{eq:lbck-bound}
\end{align}
where $D_t, t \in [T],$ is as in Theorem~\ref{thm:deda-exp3}.
In addition, if $\hell_t, t \in [T],$ is given by Algorithm~\ref{alg:delay-and-data-adaexp} using a non-increasing sequence of $\gamma_t$ values, then for all $t \in [T]$,
\begin{align}
    \sum_{i=1}^{K}
    \lfwd_{1:t,i}
    \le
    &
    \sum_{i=1}^{K}
    \sum_{s: s+d_s < t}
    \lbck_{s,i}
    + 
    \frac
    {4 {d^\star_t}^2+6d^\star_t+2} {\gamma_t}
    \,
    .
    \label{eq:lfwd-bound}
\end{align}
\end{lemma}

Next, we show that using \eqref{eq:adaptive-step-mixed-delay-and-data} results in at most a lower-order penalty on the ideal regret bound given by \eqref{eq:delay-data-adaptive-bound-aaai-IDEAL}. This is captured by the next theorem.
\begin{theorem}
[Adapting to delay and data]
\label{thm:deda-exp3}
If the step-sizes are given by \eqref{eq:adaptive-step-mixed-delay-and-data}, then \dedaexp satisfies
\begin{align}
    \sum_{t=1}^{T} \left(\EE{\ell_{t,A_t}} - \ell_{t, A^\star}\right)
    \le 
    C_T
    +
    c
    \sqrt{
    \log(K)
    \sum_{t = 1}^{T}
     \left(\sum_{s \in O_t \cup D_t} \EE{\ell_{s,A_s}\ell_{t,A_s}}
    + \sum_{i=1}^K \ell_{t,i}
    \right)    }
    \, ,
    \label{eq:delay-data-adaptive-bound-aaai_f}
\end{align}
where $C_T =
    4 (d^\star_T)^2
    + 6 d^\star_T
    + 2$,
    $c = 2+\sqrt{2}$,
    and
    $D_t = \{s: t \in O_s\}$ 
    is the set of time steps at which the feedback $\ell_{t,A_t}$ is delayed.
\end{theorem}
\begin{remark} \em
A bound that depends on $L_{T,A^*}$ instead of $\sum_i L_{T,i}$ would be preferable, but to our knowledge such bounds are not available for \ExpThree even in the non-delayed case and require other techniques (such as using the log-barrier regularizer), and even to get $\sum_i L_{T,i}$ in the \emph{final} bound requires some form of data-adaptive tuning (which in our case is further complicated by the delays).
Nevertheless, the above bound still preserves the separation of the cumulative delay and $K$. Using that $s \in D_t$ is equivalent to $t \in O_s$, we can rewrite the square-root term in the above bound as
\begin{align}
\label{eq:delay-data-rewrite}
\sqrt{\log(K)\sum_{t=1}^T
\left(
\sum_{s=t+1}^{t+d_t}
(\EE{\ell_{s,A_s}\ell_{t,A_s} + \ell_{s,A_t}\ell_{t,A_t}})
+ \sum_{i=1}^K \ell_{t,i}\right)}~.
\end{align}
Bounding the losses by $1$, this gives a regret upper bound
\[
R_T \le
C_T + c \sqrt{\log(K) \left(KT + 2\sum_{t=1}^T d_t\right)}~.
\]
However, $\sum_i L_{T,i}$ can be much smaller than $KT$, e.g., $\sum_i L_{T,i} \ll KT$ when most arms have a small loss or when the actual loss range is $[0,B]$ for some unknown $B\ll 1$ (i.e., the algorithm adapts to the unknown $B$, and the final bound depends only on $B$ and $d_t$, not on the algorithms' choices). Finally, we can obtain a novel bound by exploiting the dependence of the bound in \eqref{eq:delay-data-rewrite} on the algorithm's choices: with 
$c' = c \sqrt{\log(K)}$, we have
\vspace{-0.1cm}
\begin{align*}
    R_T
    &
    \le
    C_T
    +
    c'
    \sqrt{
    \sum_{t = 1}^{T}
    \left(
     2 d^\star_T \EE{\ell_{t,A_t}}
    + \sum_{i=1}^K \ell_{t,i}
    \right)}
    \\
    &
    \le
    C_T
    +
    c'
    \sqrt{
    2 d^\star_T
    (R_T + L_{T,A^*})
    }
    +
    c'
    \sqrt{
    \sum_{t = 1}^{T}
    \sum_{i=1}^K \ell_{t,i}
    }
    \\[-0.4em]
    &
    \le
    C_T
    +
    2
    c'
    \sqrt{
    \frac{d^\star_T}{2}
    R_T
    }
    +
    2
    c'
    \sqrt{
    \frac{d^\star_T}{2}
    L_{T,A^*}
    }
    +
    c'
    \sqrt{
    \sum_{t = 1}^{T}
    \sum_{i=1}^K \ell_{t,i}
    }
    \\
    &\
    \le
    C'_T
    +
    \frac{R_T}{2}
    +
    2
    c'
    \sqrt{
    \frac{d^\star_T}{2}
    L_{T,A^*}
    }
    +
    c'
    \sqrt{
    \sum_{t = 1}^{T}
    \sum_{i=1}^K \ell_{t,i}
    }~,\\[-1.5em]
\end{align*}
where $C'_T =  C_T + c'^2 d^\star_T$, and the last step uses $2\sqrt{ab} \le a + b$. Hence, \emph{the effect of the delay on the regret scales only with the loss of the optimal arm}: 
\[
R_T \le 2 C'_T + 4c'\sqrt{d^\star_T L_{T,A^*}/2} + 2 c'\sqrt{\sum_{i=1}^K L_{T,i}}.
\]
\end{remark}

\begin{proof}[Proof of Theorem~\ref{thm:deda-exp3}]
The proof follows the same lines as the proof of Theorem~\ref{thm:simple}, but instead of bounding the losses by their maximum value, we approximate the adaptive data-dependent step-size that needs to be used to obtain a data-adaptive bound.

First, it can be easily seen that with $\gamma_t$ and $\eta_t$ as defined, we have%
\footnote{Selecting $\eta_t=\gamma_t$ to make the first inequality can be achieved by $\eta_t^{-1}=\frac
    {4 (d^\star_t)^2 + 6 d^\star_t + 2}{2\log(K)}+\sqrt{\frac{(4 (d^\star_t)^2 + 6 d^\star_t + 2)^2}{4\log^2(K)} + \frac{\sum_{i=1}^{K}
    \sum_{s+d_s < t} \lbck_{s, i}}{\log(K)}}$,
    which would yield a slightly better but slightly uglier bound in the theorem. Our current choice of $\eta_t$ is a little smaller than the inverse of this value, hence the inequality.}
\begin{align*}
    \eta_t
    &
    \le
    \sqrt{
    \log(K)
    \left(
    \frac
    {4 (d^\star_t)^2 + 6 d^\star_t + 2}
    {\gamma_t}
    +
    \sum_{i=1}^{K}
    \sum_{s+d_s < t} \lbck_{s, i}
    \right)^{-1}
    }
    \le
    \sqrt{
    \log(K)
    \left(\sum_{i=1}^{K}
    \lfwd_{1:t,i}\right)^{-1}
    }
    \,
    ,
\end{align*}
where the second step follows from Lemma~\ref{lem:step-size-control}
.
Combining this with \eqref{eq:basic-bound-on-delay-penalty} and using the well-know AdaGrad lemma (e.g., \citep[Lemma~4]{mcmahan2017survey}),
we obtain
\begin{align*}
    \sum_{t=1}^{T}
    {\sum_{i=1}^{K} \hell_{t,i}^\top p_{t,i} \left(1 - \frac{\tl{p}_{t+1,i}}{p_{t,i}} \right)}
    &\le
    {\sum_{t=1}^{T}
    \sum_{i=1}^{K} 
    \eta_t \lfwd_{t,i}
    } \\
    &\le
    2 {
    \sqrt{
    \log(K)
    \sum_{i=1}^{K} \lfwd_{1:T,i}
    }
    }
    =
    2 {
    \sqrt{
    \log(K)
    \sum_{i=1}^{K}
    \sum_{t = 1}^{T}
    \sum_{s \in O_t \cup \{t\}}
    \hell_{s,i}
    \hell_{t,i}
    p_{t,i}}
    }
    \\
    & 
    \le
    2 {
    \sqrt{
    \log(K)
    \sum_{i=1}^{K}
    \sum_{t = 1}^{T}
    \sum_{s \in O_t \cup \{t\} \cup D_t}
    \hell_{s,i}
    \hell_{t,i}
    p_{t,i}}
    }
    \,
    .
\end{align*}

Next, from \eqref{eq:full-info-regret}, we have
\begin{align*}
    {\sum_{t=1}^{T} \hell_{t}^\top (\tl{p}_{t+1} - p^\star)} 
    &\le 
    \log(K)
    \left(
    \frac
    {
    4 (d^\star_T)^2
    + 6 d^\star_T
    + 2
    }
    {\log(K)}
    +
    \sqrt{
    \frac
    {\sum_{i=1}^{K}
    \sum_{s+d_s < T} \lbck_{s, i}
    }
    {\log(K)}
    }
    \right)
    \\
    &
    =
    C_T
    +
    {
    \sqrt{
    \log(K)
    \sum_{i=1}^{K}
    \sum_{s+d_s < T} \lbck_{s, i}
    }
    }
    \\
    &
    \le
    C_T
    +
    {
    \sqrt{
    2
    \log(K)
    \sum_{i=1}^{K}
    \sum_{t = 1}^{T}
    \sum_{s \in O_t \cup \{t\} \cup D_t}
    \hell_{s,i}
    \hell_{t,i}
    p_{t,i}
    }
    }
    \, ,
\end{align*}
where the third step follows by Lemma~\ref{lem:step-size-control}. Putting everything together, taking expectation and moving it inside the square root by Jensen's inequality, we obtain
\begin{align}
    \sum_{t=1}^{T} \left(\EE{\ell_{t,A_t}} - \ell_{t, A^\star}\right)
    \le 
    C_T
    +
    c
    \sqrt{
    \log(K)
    \sum_{t = 1}^{T}
    \sum_{s \in O_t \cup \{t\} \cup D_t}
    \EE{
    \sum_{i=1}^{K}
    \hell_{s,i}
    \hell_{t,i}
    p_{t,i}}
    }
    \, .
    \label{eq:delay-data-adaptive-bound-aaai}
\end{align}
What remains is to work out the expectations. To do this, notice that in our algorithm any action only affects another future action if the corresponding feedback arrives before the second action is taken. Therefore, whenever $s \in O_t \cup D_t$, the indicators $I_{s,i}$ and $I_{t,i}$ are independent given $\cH_{\max\{s,t\}}$, and for such $t$ and $s$,
\vspace{-0.2cm}
\begin{align*}
\EE{\sum_{i=1}^{K} \hell_{s,i} \hell_{t,i} p_{t,i}} 
&\le
\EE{\sum_{i=1}^{K} 
(\ell_{s,i}\ell_{t,i} I_{t,i})
\frac{I_{s,i}}{p_{s,i}}}
= \EE{\sum_{i=1}^{K} \ell_{s,i}\ell_{t,i} \EE{I_{t,i} \frac{I_{s,i}}{p_{s,i}} \Big| \cH_{\max\{s,t\}}}} \\
& = \EE{\sum_{i=1}^{K} \ell_{s,i}\ell_{t,i} p_{t,i}}
= \EE{\ell_{s,A_t}\ell_{t,A_t}}
\,,
\end{align*}
where the first step follows by the definition of $\hat{\ell}_j$, the second by the definition of $I_{t,i}$ and the tower rule, 
and the third by the fact that $\EE{I_{t,i} I_{s,i}/{p_{s,i}} \big| \cH_{\max\{s,t\}}} = \EE{I_{t,i} | \cH_{\max\{s,t\}}} \EE{I_{s,i} | \cH_{\max\{s,t\}}} / p_{s,i} = p_{t,i}$ whenever $s \in O_t \cup D_t$.
Furthermore, when $s=t$, the corresponding term is
$\EE{\hell_{t,i}^2 p_{t,i}}=\EE{\ell_{t,i}^2 I_{t,i}/p_{t,i}} \le \ell_{t,i}. 
$
Substituting these in \eqref{eq:delay-data-adaptive-bound-aaai} completes the proof.
\end{proof}

\section{Conclusions}

In this paper we presented delay- and data-adaptive algorithms for the multi-armed bandit problem with delayed feedback. 
First, through a remarkably simple proof technique, we showed that the expected regret of our simpler algorithm, \adaexp, scales optimally with the sum of the delays, up to logarithmic factors (without any advance knowledge of the delays). We also showed that if the  implicit-exploration loss estimate of \citet{neu2015explore} is used, \adaexp achieves the same  near-optimal regret guarantees with high probability, providing the first high-probability regret bound in the literature for a fully delay-adaptive bandit algorithm.

One problem with the regret bounds that scale with the sum of the delays is that they become too large when individual delays are large, for example, a single delay of $T$ has a significant impact on the regret bound. Recently, \citet{thune19} addressed this question by ``skipping'' rounds with large delays, significantly reducing the regret. However, to achieve this, they needed to know the delays at action-time. \citet{ZiSe19} provided a delay-adaptive solution for this problem. Based on the latter result, we proved similar ``skipping'' regret bounds for modified versions of \adaexp, both in expectation and with high probability.

Finally, we presented the \dedaexp algorithm, the first method for delayed bandits that, besides the delays, also adapts to the losses, achieving a potentially large improvement on easy problems. While for \adaexp, a bound on the expected regret was possible with the standard importance-weighting loss estimator, and the estimator based on implicit exploration was only needed for the high probability bound, employing the latter in \dedaexp is crucial for being able to to control $\eta_t$ properly, and hence for obtaining a meaningful regret bound even in expectation.
Deriving high-probability regret bounds and extending the skipping technique to \dedaexp is left for future work. Solving these problems require some innovations: For the first one, new results concerning the concentration of  products of certain loss estimates are needed. The issue with the second problem is that the natural data-dependent variant of the skipping decision (rather than the version used together with \adaexp, which only depends on the delays, but not on the observed losses) induces a complicated dependence on past actions, significantly complicating the simple deterministic skipping mechanism which we used and analyzed for \adaexp.

\section*{Acknowledgements}

The authors would like to thank Tor Lattimore for his insightful comments.
 
\bibliographystyle{plainnat}
\bibliography{refs,refsaaai}

\newpage
\appendix

\section{Proof of Lemma~\ref{lem:step-size-control}}
\label{sec:proof-of-lemma}
\begin{proof}
We start by fixing $t \in [T]$ and expanding the sum $\lfwd_{1:t,i}$:
\begin{align*}
    \lfwd_{1:t, i}
    & =
    \sum_{j = 1}^{t}
    \sum_{s = 1}^{t}
    \I{s \le j \le s+d_s}
    \hell_{s,i}
    \hell_{j,i}
    p_{j,i}
    \\
    & =
    \sum_{s = 1}^{t}
    \sum_{j = 1}^{t}
    \I{s \le j \le s+d_s \le j + d_j}
    \hell_{s,i}
    \hell_{j,i}
    p_{j,i}
    +
    \sum_{s = 1}^{t}
    \sum_{j = 1}^{t}
    \I{s \le j \le j + d_j < s+d_s}
    \hell_{s,i}
    \hell_{j,i}
    p_{j,i}
    \\
    & =
    \sum_{s = 1}^{t}
    \sum_{j = 1}^{t}
    \I{j \le s+d_s \le j + d_j}
    \hell_{s,i}
    \hell_{j,i}
    p_{j,i}
    +
    \sum_{s = 1}^{t}
    \sum_{j = 1}^{t}
    \I{s \le j + d_j \le s+d_s}
    \hell_{s,i}
    \hell_{j,i}
    p_{j,i}
    - S_{t,i}
    \,
    ,
\end{align*}
where 
\begin{align*}
    S_{t,i} 
    &= 
    \sum_{s = 1}^{t}
    \sum_{j = 1}^{t}
    \I{s > j, j \le s+d_s \le j + d_j}
    \hell_{s,i}
    \hell_{j,i}
    p_{j,i}
    +
    \sum_{s = 1}^{t}
    \sum_{j = 1}^{t}
    \I{s > j, s \le j + d_j < s+d_s}
    \hell_{s,i}
    \hell_{j,i}
    p_{j,i}
    \\
    &\qquad
    +
    \sum_{s = 1}^{t}
    \sum_{j = 1}^{t}
    \I{j + d_j = s+d_s}
    \hell_{s,i}
    \hell_{j,i}
    p_{j,i}
    \,
    .
\end{align*}
Moving $S_{t,i}$ to the left,
\begin{align*}
    \lfwd_{1:t, i}
    + S_{t,i}
    & =
    \sum_{s = 1}^{t}
    \sum_{j = 1}^{t}
    \I{j \le s+d_s \le j + d_j}
    \hell_{s,i}
    \hell_{j,i}
    p_{j,i}
    +
    \sum_{s = 1}^{t}
    \sum_{j = 1}^{t}
    \I{s \le j + d_j \le s+d_s}
    \hell_{s,i}
    \hell_{j,i}
    p_{j,i}
    \\
    & =
    \sum_{s = 1}^{t}
    \sum_{j = 1}^{t}
    \I{s \le j+d_j \le s + d_s}
    \hell_{s,i}
    \hell_{j,i}
    (p_{j,i} + p_{s,i})
    \\
    & =
    \sum_{s = 1}^{t}
    \sum_{j = 1}^{t}
    \I{s \le j+d_j \le s + d_s < t}
    \hell_{s,i}
    \hell_{j,i}
    (p_{j,i} + p_{s,i})
    \\
    &\qquad
    +
    \sum_{s = 1}^{t}
    \sum_{j = 1}^{t}
    \I{s \le j+d_j \le s + d_s, s+d_s \ge t}
    \hell_{s,i}
    \hell_{j,i}
    (p_{j,i} + p_{s,i})
    \\
    &=
    \sum_{s: s+d_s < t}
    \sum_{j: s \le j+d_j \le s+d_s}
    \hell_{s,i}
    \hell_{j,i}
    (p_{j,i} + p_{s,i})
    \\
    &\qquad
    +
    \sum_{s = 1}^{t}
    \sum_{j = 1}^{t}
    \I{s \le j+d_j \le s + d_s, s+d_s \ge t}
    \hell_{s,i}
    \hell_{j,i}
    (p_{j,i} + p_{s,i})
    \,
    ,
\end{align*}
where the second step follows by swapping the names of $s$ and $j$ in the first sum on the r.h.s. Therefore,
\begin{align}
    \lfwd_{1:t, i}
    + S_{t,i}
    & =
    \sum_{s: s+d_s < t}
    \lbck_{s,i}
    +
    M_{t,i}
    \label{eq:step-control-main-inequality}
    \,
    ,
\end{align}
where $M_{t,i} = \sum_{s = 1}^{t}
\sum_{j = 1}^{t}
\I{s \le j+d_j \le s + d_s, s+d_s \ge t}
\hell_{s,i}
\hell_{j,i}
(p_{j,i} + p_{s,i})$. 

To continue, note that $M_{t,i}$ and $S_{t,i}$ are non-negative. Hence, to get the results of the lemma it remains to bound the terms $M_{t,i}$ and $S_{t,i}$ from above.
To that end, note 
\begin{align*}
    S_{t,i} &= 
    \sum_{s = 1}^{t}
    \sum_{j = 1}^{t}
    \I{s > j, s \le j+d_j}
    \hell_{s,i}
    \hell_{j,i}
    p_{j,i}
    +
    \sum_{s = 1}^{t}
    \sum_{j = 1}^{t}
    \I{j + d_j = s+d_s}
    \hell_{s,i}
    \hell_{j,i}
    p_{j,i}    \,
    \\
    &\le
    \sum_{j = 1}^{t}
    \sum_{s: j \in O_s}
    \hell_{s,i}
    \hell_{j,i}
    p_{j,i}
    +
    \sum_{j = 1}^{t}
    \sum_{s \in O_j \cup \{j\} \cup D_j}
    \hell_{s,i}
    \hell_{j,i}
    p_{j,i}
    \,
    ,
\end{align*}
where in the first step we have 
merged the first two summations, and split the last one, in the definition of $S_{t,i}$. The second step then follows from using the definition of $O_s$ for the first term, and noting for the second term that either $j =s $, or $j < s$ (which, together with $j + d_j = s+d_s \ge s > j$, implies $j \in O_s$) or $j > s$ (which, together with $s+d_s = j+d_j \ge j > s$ implies $s \in O_j$).

Combining with the definition of $\lfwd_{1:t,i}$ and recalling the definition of $D_j$, we obtain
\begin{align*}
    \sum_{s: s+d_s < t}    \lbck_{s,i}
    & \le
    \sum_{s: s+d_s < t}    \lbck_{s,i}
    + M_{t,i}
    \\
    &
    =
    \lfwd_{1:t,i}
    +
    S_{t,i}
    \\
    &
    \le
    \lfwd_{1:t,i}
    +
    \sum_{j = 1}^{t}
    \sum_{s \in D_j}
    \hell_{s,i}
    \hell_{j,i}
    p_{j,i}
    +
    \sum_{j = 1}^{t}
    \sum_{s \in O_j \cup \{j\} \cup D_j}
    \hell_{s,i}
    \hell_{j,i}
    p_{j,i}
    \\
    &
    =
    2
    \sum_{j = 1}^{t}
    \sum_{s \in O_j \cup \{j\} \cup D_j}
    \hell_{s,i}
    \hell_{j,i}
    p_{j,i}
    \,.
\end{align*}
This concludes the proof of the first inequality \eqref{eq:lbck-bound}.

To bound $M_{t,i}$, note that the summations run up to $t$, and the conditions
$\I{s+d_s \ge t}$
and
$\I{s \le j+d_j}$
imply, respectively, that the value of the sum is zero for $s < t-d^\star_t$ and $j < t - 2d^\star_t$ (since $d_j$ and $d_s$ are at most $d^\star_t$). Hence, we have
\begin{align*}
    \sum_{i=1}^{K} M_{t,i}
    &\
    \le
    \sum_{s = \max\{1, t-d^\star_t\}}^{t}
    \sum_{j = \max\{1, t-2 d^\star_t\}}^{t}
    \sum_{i=1}^{K}
    \left(
    \frac{\ell_{s,i} I_{s,i}}{p_{s,i}+\gamma_s}
    \ell_{j,i} I_{j,i}
    + 
    \ell_{s,i} I_{s,i}
    \frac{\ell_{j,i} I_{j,i}}{p_{j,i}+\gamma_j}    \right)
    \\
    &\
    \le
    \sum_{s = \max\{1, t-d^\star_t\}}^{t}
    \sum_{j = \max\{1, t-2 d^\star_t\}}^{t}
    \sum_{i=1}^{K}
    \left(
    \frac{I_{j,i} I_{s,i}}{\gamma_t}
    + 
    \frac{I_{s,i} I_{j,i}}{\gamma_t}
    \right)
    \,
    ,
\end{align*}
using the fact that the losses are non-negative and upper-bounded by $1$, the definition of $\hell_k, k \in [T]$, and the fact that $\gamma_t$ is a non-increasing sequence in $t$.
Hence, $\sum_{i=1}^{K} M_{t,i} \le (4 {d^\star_t}^2+6d^\star_t+2) / \gamma_t$. Putting back in \eqref{eq:step-control-main-inequality} and summing over $i$ completes the proof of the second inequality \eqref{eq:lfwd-bound}.
\end{proof}

\end{document}